\documentclass[11pt, a4paper]{article}
\usepackage[utf8]{inputenc}

\usepackage{sansmathfonts}
\usepackage[T1]{fontenc}
\usepackage[
        activate={true,nocompatibility},
        final,
        kerning=true,
        spacing=true
        ]{microtype}
\microtypecontext{spacing=nonfrench}
\usepackage{lmodern}

\usepackage[left=0.7in, right=0.7in, top=1in, bottom=1in]{geometry}

\usepackage{tikz}      
\usetikzlibrary{positioning,calc}

\usepackage{hyperref}  
\usepackage{amsmath}   
\usepackage{amssymb}   
\usepackage{graphicx}  
\usepackage{enumitem}  
\usepackage{authblk}   
\usepackage{amsthm}
\newtheorem{theorem}{Theorem}

\newtheorem{lemma}[theorem]{Lemma}
\newtheorem{remark}[theorem]{Remark}
\newcommand*{\latinabbr}[1]{#1}
\newcommand*{\ie}{\latinabbr{i.e.}} 

\DeclareMathOperator*{\argmin}{arg\,min}

\title{Generative models and Bayesian inversion using Laplace approximation}
\author{Manuel Marschall}
\author{Gerd W\"ubbeler}
\author{Franko Schm\"ahling}
\author{Clemens Elster}
\affil{Physikalisch-Technische Bundesanstalt,
       Abbestra\ss e 2-12,
       10587 Berlin, Germany}
\date{Date: \today}

\begin{document}

\maketitle

\begin{abstract}
The Bayesian approach to solving inverse problems relies on the choice of a prior. This critical ingredient allows the formulation of expert knowledge or physical constraints in a probabilistic fashion and plays an important role for the success of the inference. Recently, Bayesian inverse problems were solved using generative models as highly informative priors. Generative models are a popular tool in machine learning to generate data whose properties closely resemble those of a given database. Typically, the generated distribution of data is embedded in a low-dimensional manifold. For the inverse problem, a generative model is trained on a database that reflects the properties of the sought solution, such as typical structures of the tissue in the human brain in magnetic resonance (MR) imaging.
The inference is carried out in the low-dimensional manifold determined by the generative model which strongly reduces the dimensionality of the inverse problem. However, this proceeding produces a posterior that admits no Lebesgue density in the actual variables and the accuracy reached can strongly depend on the quality of the generative model.
For linear Gaussian models we explore an alternative Bayesian inference based on probabilistic generative models which is carried out in the original high-dimensional space.
A Laplace approximation is employed to analytically derive the required prior probability density function induced by the generative model. Properties of the resulting inference are investigated. Specifically, we show that derived Bayes estimates are consistent, in contrast to the approach employing the low-dimensional manifold of the generative model.
The MNIST data set is used to construct numerical experiments which confirm our theoretical findings.
It is shown that the proposed approach can be advantageous when the information contained in the data is high and a simple heuristic is considered for the detection of this case. Finally, pros and cons of both approaches are discussed.

\textbf{keywords:} \textit{Bayesian inference, Asymptotic properties of parametric estimators, Generative models, Machine learning, Laplace approximation}
\end{abstract}

\section{Introduction}
Inverse problems are ubiquitous and statistical methods for their treatment have been developed for a long time \cite{kaipio2006statistical,bissantz2008statistical,draper1998applied}.
Frequently, inverse problems are considered in a discretized form, resulting in (large-scale) regression tasks, or they are posed as discrete inverse problems from the start, for example in functional MR imaging \cite{smith2007spatial,lee2014spatial}. Inverse problems are often ill-posed, or ill-conditioned in the discrete case, and reliable estimation requires some form of regularization such as Tikhonov regularization \cite{engl1996regularization}. Bayesian inference \cite{robert2007bayesian,gelman1995bayesian} provides an alternative approach to ill-posed inverse problems, in which the employed prior renders the estimation task well-posed. For example, a Gaussian prior can lead to a maximum a posteriori (MAP) estimate equivalent to the result of a Tikhonov regularization. Gaussian Markov random field (GMRF) priors \cite{rue2005gaussian} are another popular class of priors used. For example, to model a priori spatial smoothness as it is often relevant in spatial modeling or image processing.

While these analytic priors have been successfully applied in many applications, they do not always adequately model the prior knowledge available in an individual problem. For example, in quantitative MR imaging of the brain employed conventional prior distributions do not truly reflect the structure of the brain~\cite{adler2018deep}. Generative models from machine learning using modern architectures such as \emph{generative adversarial networks} (GANs)~\cite{goodfellow2014generative} or \emph{variational auto-encoders} (VAEs)~\cite{kingma2019introduction}, on the other hand, have proven to generate data that closely resemble the properties of data in a training set.
Their efficiency, adaptability and easy accessibility in standard libraries result in a multitude of applications. From speech synthesis~\cite{saito2017statistical} over text generation~\cite{wang2018text} to molecule generation~\cite{hong2019molecular} and urbanization~\cite{albert2018modeling}, to mention but just a few.
For a comprehensive overview on the usage of deep learning methods for the solution of linear inverse problems we refer to~\cite{bai2020deep}, and for a review on the application of data-driven models  to~\cite{arridge2019solving}. For an overview of generative models in computer vision see for instance~\cite{cao2018recent,park2021review,yangjie2018review}, and for a review in medical imaging cf.~\cite{yi2019generative}.

In view of their capabilities, generative models have recently been considered as priors in a Bayesian treatment of inverse problems, cf., e.g., \cite{arridge2019solving}.
The advantage of such proceeding is that individual properties of the problem at hand, such as the structure of brain images, are adequately taken into account~\cite{adler2018deep}, which leads to an improved inference. Another advantage is that often the sought (discrete) function or field belongs to a low-dimensional manifold, which is exploited by current generative models such as GANs or VAEs. Then, the inference can be carried out in a low-dimensional space of latent variables, which strongly facilitates the calculation of the results in a Bayesian inference, cf. \cite{10.1002/essoar.10501256.1,griffiths2020constrained,mucke2021markov,bora2017compressed}.

However, restricting the inference to a low-dimensional latent space has the disadvantage that it admits no Lebesgue density for the posterior in the high-dimensional space of the actual variables~\cite{holden2021bayesian}. Furthermore, the estimation accuracy reached can strongly depend on the quality of the generative model~\cite{tripp2020sample}. To circumvent this drawback and somehow escape from the lower dimensional manifold,~\cite{holden2021bayesian} takes the mean of the push-forward of the posterior in the latent space as a Bayes estimate. Some novel approaches also incorporate the inversion problem directly into the learning process~\cite{hussein2020image}, e.g. to achieve super resolution or to reconstruct high fidelity magnetic resonance images~\cite{sood2018application,bhadra2020medical}. In this regard, we like to mention~\cite{adler2018deep} for their \emph{deep direct estimation} procedure and the development of a \emph{conditional Wasserstein GAN discriminator}, which allows sampling from the posterior.

In this paper, we consider a Bayesian treatment of linear Gaussian inverse problems that is carried out in the high-dimensional original variable space. The employed prior is determined by a probabilistic generative model. We propose an analytic solution to this task based on a Laplace approximation of the prior which is inspired by a treatment of~\cite{Liue2101344118} in the context of density estimation. The choice of the class of considered inverse problems is made in view of tasks such as inpainting, restoration of images~\cite{andrews1977digital} or medical imaging~\cite{bhadra2020medical,Kofler_2020}. Another motivation for our choice of problems is that it allows analytical results to be derived.

The properties of the proposed inference such as, for example, consistency, are explored and contrasted to those obtained for the inference which is carried out in the low-dimensional latent space. In addition, an image restoration task for blurred and noisy MNIST data~\cite{deng2012mnist} is taken to quantitatively compare the two approaches. The comparison is augmented by a simple heuristic for the choice of method. Pros and cons of both approaches are discussed in view of our results.

The paper is organized as follows. In Section~\ref{sec:generator and inverse problem} the considered class of inverse problems is specified and generative models are introduced. The Bayesian treatment of the inverse problems utilizing a generative model as a prior is then considered in Section~\ref{sec:BI}. After recalling the inference utilizing the low-dimensional latent space and characterizing its properties, the proposed approach using a probabilistic generator is developed and explored. Subsequently, the two approaches are discussed and assessed in terms of their properties. Our treatment assumes knowledge about the variance of the observations, and finally, we briefly note on the possible generalization to the case of unknown variance. In Section~\ref{sec:numerics} numerical examples are presented for a quantitative assessment. An outlook on potential future research and conclusions from our findings are given in Section~\ref{sec:outlook}.

\section{Generative models and considered class of inverse problems}
\label{sec:generator and inverse problem}
This section specifies the considered class of inverse problems and introduces generative models. The generative models will later be used for the construction of a prior in a Bayesian treatment of the inverse problems. Two classes of generative models are distinguished: one that uses a \emph{deterministic} map applied to the latent variables and a \emph{probabilistic} approach. The former will be taken to carry out a Bayesian inference in a low-dimensional latent space, the latter serves as the starting point for the proposed inference carried out directly in the high-dimensional space of the actual variables.

\subsection{Linear inverse problems}
We consider linear inverse problems of the form
\begin{equation}\label{eq2.1}
    y|x \sim \mathcal N(Ax, \sigma^2I),
\end{equation}
where $A$ denotes some given operator mapping from the (original) space $\mathcal X$ to the space of observables $\mathcal Y$. We take
$\mathcal X = \mathbb{R}^d$ and $\mathcal Y = \mathbb{R}^n$ so that $A$ is an $n$ by $d$ matrix. We assume throughout that $A$ has full rank and $n \ge d$. The goal is to infer $x$ given the data $y$. The variance $\sigma^2$ is mostly treated as a fixed given parameter and is largely suppressed in our notation, except for those cases in which it is explicitly included in the inference.

Problems of the form \eqref{eq2.1} typically emerge from the discretization of a continuous inverse problem in which $x$ could be a spatially distributed property of the human body or an image that has been blurred. Usually, the dimension of the (discretized) $x$ will be high and the linearity assumption simplifies matters. The assumption made about the structure of the covariance matrix in the sampling distribution \eqref{eq2.1} essentially means that the full covariance matrix needs to be known up to a factor, a situation which, after a suitable linear transformation, yields a problem of the form in \eqref{eq2.1}.
We note that simple models of the form \eqref{eq2.1} are relevant in many applications, for example functional MR imaging~\cite{boynton1996linear}, inpainting~\cite{richard2001fast} or de-blurring of images~\cite{carasso1999linear}.

\subsection{Generative models}
Generative models such as VAEs or GANs can produce random samples by transforming a simple distribution of latent variables, e.g. a multivariate standard Gaussian distribution, through a neural network. Those networks are trained on a (large) database in such a way that the resulting distribution approximates the distribution underlying the employed database. The database could for instance consist of a training set of images with a specific characteristic for the task at hand, e.g. handwritten digits from the MNIST data set or a sequence of MR images from various patients. Often, such data sets can be modeled as belonging to a low-dimensional manifold $\mathcal{M}\subset\mathcal{X}$ which is exploited by the generative models through a correspondingly chosen low dimension of the latent variables.

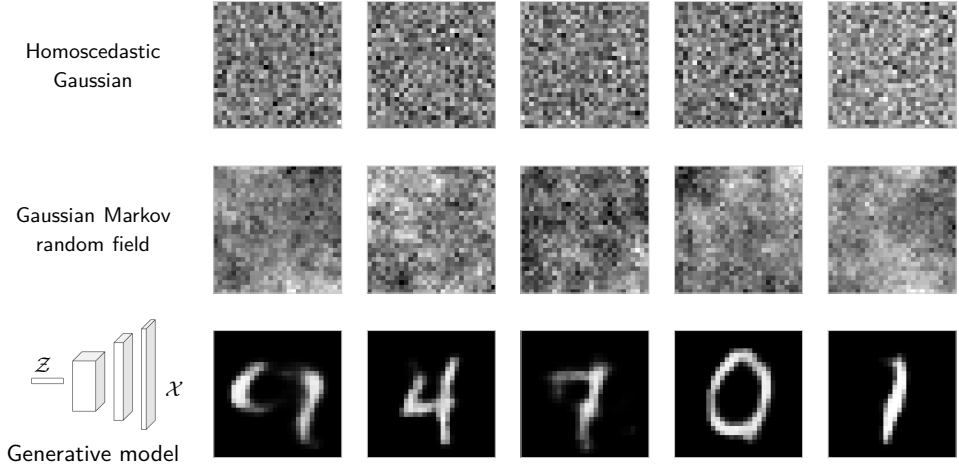
\begin{figure}
    \centering
\begin{tikzpicture}[thick,scale=0.8, every node/.style={transform shape}]
    \node[anchor=north west] (imagenormal){
        \includegraphics[width=.7\textwidth]{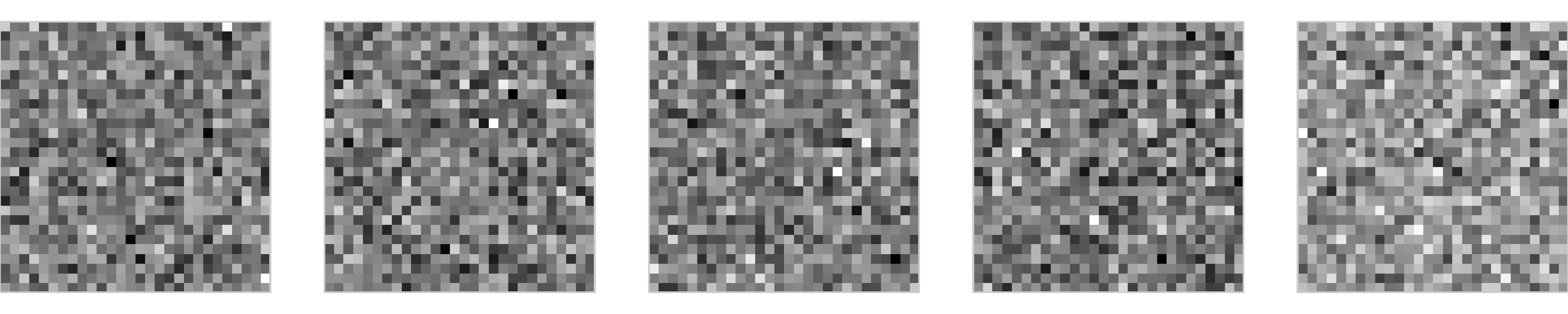}
    };
    \node[anchor=south west, left= 0.2cm of imagenormal, text width=3cm, align=center] (textnormal) {
       \small Homoscedastic Gaussian
    };
    \node[anchor=north west, below= 0cm of imagenormal] (imagegmrf) {
        \includegraphics[width=.7\textwidth]{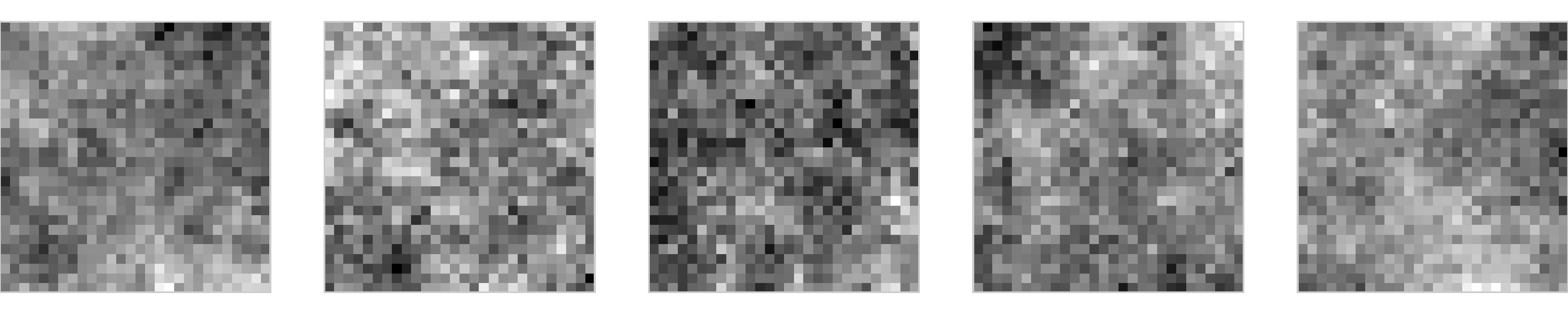}
    };
    \node[anchor=south west, left=0.2cm of imagegmrf, text width=3cm, align=center] (textgmrf) {
        \small Gaussian Markov random field
    };
    \node[anchor=north west, below=0.8cm of textgmrf, text width=3cm, align=center] (generator) {
        \includegraphics[width=.7\textwidth]{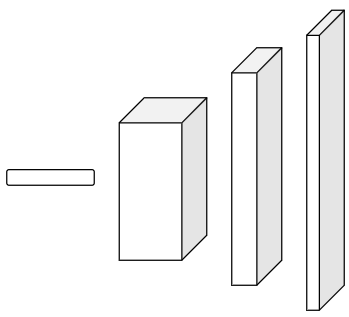}

        Generative model
    };
    \node[below left=-2cm and -1.1cm of generator](zspace) {\small $\mathcal{Z}$};
    \node[below right=-1.6cm and -0.6cm of generator](xspace) {\small $\mathcal{X}$};
    \node[anchor=north west, below=0cm of imagegmrf] (x) {
        \includegraphics[width=.7\textwidth]{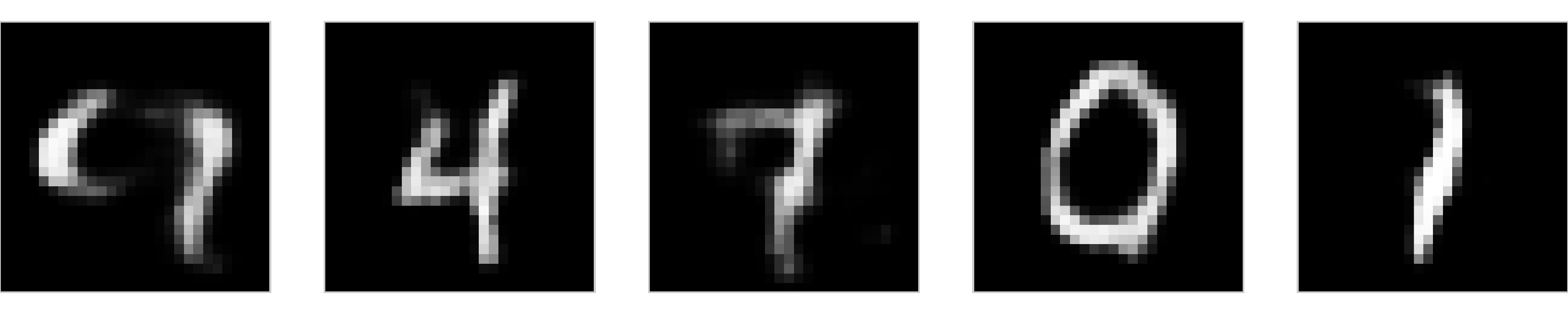}
    };
\end{tikzpicture}
    \caption{Different variants of prior knowledge visualized by samples: (top row) a standard Gaussian model $\mathcal N(0, I)$, (mid row) realizations of a Gaussian Markov random field (GMRF) with eight nearest neighbor interaction and (bottom row) realizations of a generative model trained on the set of handwritten digits (MNIST). All images are scaled to $[0, 1]$.}
    \label{fig:sketch_model}
\end{figure}

In Figure~\ref{fig:sketch_model} we depict the considered structure of the generative model. A latent vector $z\in\mathcal{Z} := \mathbb{R}^p$ of lower dimension $p\ll d$ is mapped by the generator to the original variable space $\mathcal{X}$. Throughout this work, we assume that such a \emph{generator} or generative model is available and that it has already been \emph{trained} on a data set, together with a multivariate standard Gaussian distribution $\pi(z) = N(z\vert 0, I)$ as prior for the latent variables.
In literature, there are usually two types of generator outputs considered. The \emph{deterministic} map, subsequently denoted by $g\colon \mathcal{X}\to\mathcal{M}$, and the \emph{probabilistic} formulation.
For the probabilistic formulation, we assume a Gaussian model according to
\begin{equation}
  \label{eq:encoder-decoder}
    \quad \pi(x\vert z) = \mathcal{N}(x \vert g(z), \Gamma(z)),
\end{equation}
where $g(z)$ and $\Gamma(z)$ denote the outputs of a trained (deep) neural network for input $z$. The covariance matrix is modeled as one of the following variants.
\begin{equation}
    \Gamma(z) = \left\{
    \begin{array}{ll}
         \lambda(z)^{-1}I, &  \text{constant with precision }  \lambda(z)>0\\
         \mathrm{diag}(\gamma_1(z), \ldots, \gamma_d(z)), & \text{diagonal with variances } \gamma_i(z)>0, i=1,\ldots, d\\
         \Gamma(z), & \text{full covariance matrix in } \mathbb{R}^{d, d}.
    \end{array}\right.
\end{equation}

The motivation for the use of generative models for the data-driven construction of a prior is that these models are extremely versatile and capable to produce a distribution whose realizations closely resemble those of the database used to train them,
cf. the randomly produced digits produced by a trained generative model in Figure 1 which resemble the properties of the MNIST digit database. When a (large) database is available whose members represent typical features of the solution of a considered inverse problem, a prior constructed by a generative model trained on that database can be highly informative and beneficial for a Bayesian solution to that problem.
A conventional prior such as a standard Gaussian prior (Figure~\ref{fig:sketch_model} top row) or a GMRF prior (Figure~\ref{fig:sketch_model} middle row), used to turn an inference into a well-posed problem and exploiting the prior knowledge of smoothness, on the other hand, will generally be much less informative, for example when considering the task to infer a digit from a blurred image of it. In fact, realizations of such a prior will not even approximately resemble a digit. Supplying a large database for a particular problem often is challenging, and techniques such as data augmentation~\cite{shorten2019survey} or virtual experiments are used in this context. However, these issues are beyond the scope of this paper.

\section{Bayesian inference using generative models}
\label{sec:BI}
A Bayesian inference is considered for inverse problems of the form \eqref{eq2.1} when using a prior that is constructed from a generative model. We start by
recalling an inference in latent space, followed by a push-forward through the deterministic mapping of the generator, as proposed in~\cite{holden2021bayesian,mucke2021markov}. Then an inference procedure is developed that works directly in the space
$\mathcal{X}$ of the actual variables by using a prior constructed from a probabilistic formulation of the generator. Finally, pros and cons of both approaches are discussed in terms of the inferential properties derived for them.

\subsection{Inference in latent space}
\label{sec:deterministic}
This approach models the data from \eqref{eq2.1} by
\begin{equation}\label{eq3.1}
    y|z \sim \mathcal N(Ag(z), \sigma^2I),
\end{equation}
where $g$ denotes the deterministic mapping of the employed generative model and the variance $\sigma^2>0$ is assumed to be known. In using the prior
\begin{equation}
\pi(z)= \mathcal{N}(z \vert 0, I)
\end{equation}
of the generative model, the posterior in latent space is proper and given by
\begin{equation}
  \label{eq:latent posterior}
    \pi(z\vert y) \propto \exp\left(-\frac{1}{2\sigma^2}\|Ag(z) - y\|_2^2\right)\pi(z).
\end{equation}
Then, a subsequent change of variables using the deterministic map $g$ generates a distribution in $\mathcal{X}$. By the fairly general form and usually nonlinear, non-invertible structure of $g$, the posterior in latent space has no closed form. A similar approach is pursuit in~\cite{10.1002/essoar.10501256.1} using a deterministic Auto-Encoder and in~\cite{holden2021bayesian} using a VAE with deterministic decoder. The latter reference correctly argues that the $g$ push-forward of the posterior in latent space does not admit a Lebesgue density in the variable space $\mathcal{X}$. We summarize these properties in the following lemma.
\begin{lemma}[Degenerated push-forward]
\label{lem:latent posterior}
Assume that the mapping $g$ satisfies some (weak) regularity conditions, for example continuity. Then the posterior in latent space~\eqref{eq:latent posterior} is proper and has finite $p$-th moment for $p<\infty$. The $g$ push-forward of $\pi(z|y)$ is a distribution in $\mathcal{X}$ which admits no Lebesgue density.
\end{lemma}
\begin{proof}
Propriety and the existence of moments for the posterior in latent space follows directly by $1$ being an upper bound for the continuous likelihood exponent and the existence of every $p$-th moment $p<\infty$ for the Gaussian prior $\pi(z)$.
The existence of no Lebesgue density on $\mathcal{X}$ follows from the fact that $g$ maps all probability mass to $\mathcal{M}$ which is assumed to have lower dimensionality than $\mathcal{X}$.
\end{proof}
Even though the $g$ push-forward of $\pi(z|y)$ does not admit a Lebesgue density, it is a well-defined distribution.
To efficiently compute statistics of the $g$ push-forward of the latent posterior, the authors in~\cite{holden2021bayesian} employ a \emph{parallel tempered, preconditioned Crank-Nicolson MCMC} to generate samples $\{z_i\}_i^N$ from the posterior in latent space $\mathcal Z$ to subsequently approximate the posterior expectation
\begin{equation}
  \label{eq:latent posterior mean}
    E_{\pi(z\vert y)}[g(z)\vert y] \approx \frac{1}{N}\sum_{i=1}^N g(z_i).
\end{equation}
This expectation does generally not belong to the image space of the generator map, but it can be expected to be close to it. Alternatively, one can numerically compute the MAP $z_{\mathrm{MAP}}$ of the posterior~\eqref{eq:latent posterior} in the latent space and then use $g(z_{\mathrm{MAP}})$ as an estimate. This choice directly illuminates the limitations of the approach. By construction, $g(z_{\mathrm{MAP}})$ is limited to the image space of the generator map.
This observation is a result of an inherent bias which is introduced by the statistical model~\eqref{eq3.1} compared to the model~\eqref{eq2.1}. Consequently, arguing from a frequentist point-of-view, the latent space approach suffers from consistency issues.
In particular, the posterior mean estimate in~\eqref{eq:latent posterior mean} viewed in dependence on the observation $y$ takes the role of an Bayes estimator for the true value $x\in\mathcal X$ under $L^2$ loss. This estimator is usually referred to the minimum mean square error (MMSE) estimator. This interpretation allows us to formulate the following lemma as the main result regarding consistency in the latent space approach.
\begin{lemma}[Inconsistency of Bayes estimator]
\label{lem:inconsistency}
Let $\widehat{g(z)} = \widehat{g(z)}(y) =  E_{\pi(z\vert y)}[g(z)\vert y]$ denote the MMSE estimator and consider  $\{\widehat{g_\sigma(z)} \}_\sigma$ explicitly dependent on the data variance of the sampling distribution~\eqref{eq3.1}.
Furthermore, assume that $g$ is continuous and such that the model~\eqref{eq3.1} is identifiable.
Let $x\in\mathcal{X}$ and assume $x$ is not contained in the image space of the generator map, \ie, there exists a $\delta>0$ such that $\min_{z\in\mathcal{Z}} \|x - g(z)\|\geq\delta$. Then, the estimator $\widehat{g_\sigma(z)}$ is not consistent, \ie, $\widehat{g_\sigma(z)}$ does not converge to $x$ in probability as $\sigma\to 0$.
\end{lemma}
\begin{proof}
  By forming an identifiable model and having a posterior in latent space such as~\eqref{eq:latent posterior} implies
  \begin{equation}
    A\widehat{g(z)_\sigma}\underset{\sigma\to 0}{\overset{p}\longrightarrow}Ag(z_0) \text{, \ie, for all }\epsilon>0, ~ \mathbb{P}(\|A\widehat{g_\sigma(z)} - Ag(z_0) \|_2 > \epsilon) \longrightarrow 0,
  \end{equation}
  where $z_0 = \argmin_{z\in\mathcal{Z}} \|Ax - Ag(z)\|_2$.
  Since $A$ is full rank, it follows $\widehat{g_\sigma(z)}\underset{\sigma\to 0}{\overset{p}\longrightarrow} g(z_0)$ and with the assumption $\|x-g(z_0)\|\geq \delta > 0$ the claim follows. In particular, for $\epsilon < \delta$ it holds $\mathbb{P}(\|\widehat{g_\sigma(z)} - x\|_2 > \epsilon) \nrightarrow 0.$
\end{proof}
\begin{remark}
\label{rem:limit}
Taking the limit with respect to $\sigma$ for the considered linear inverse problem is equivalent to taking the limit with respect to infinitely many repeated observations of $y$. This can be seen by the fact that the mean of $y$ in the likelihood is a sufficient statistic with variance $\sigma^2/k$, where $k$ is the number of repetitions.
\end{remark}
\begin{remark}
The result above generalizes to more general estimators defined by a possible different loss, which induces a different topology in $\mathcal{X}$.
\end{remark}

Having the asymptotic behavior of the estimator established, we are also interested in the asymptotic covariance of the $g$ push-forward of the latent posterior.
\begin{lemma}[Asymptotic covariance]
\label{lem:covariance push forward}
Assume that $g$ is continuous and renders the model~\eqref{eq3.1} identifiable. Then, there exists $z^\ast\in\mathcal{Z}$ such that the latent posterior~\eqref{eq:latent posterior} converges in total variation norm to the Dirac measure centered at $z^\ast$. Moreover, assume that $g$ is totally differentiable in $z^\ast$. Then, the asymptotic covariance of the $g$ push-forward of the latent space posterior is given by the inverse of the Fisher information matrix at $z^\ast$
\begin{equation}
    \check{C} = J_{z^\ast}\left(\sigma^{-2} J_{z^\ast}^T A^T A J_{z^\ast}\right)^{-1}J_{z^\ast}^T,
\end{equation}
where $J_{z^\ast}$ denotes the Jacobi matrix of $g$ evaluated at $z=z^\ast$.
\end{lemma}
\begin{proof}
The model~\eqref{eq3.1} and the latent posterior~\eqref{eq:latent posterior} fulfill the assumptions of the Bernstein-von-Mises theorem. Hence, the convergence of the posterior to a Gaussian can be assessed and $\check C$ follows by simple calculus. Taking the second derivative of the log-likelihood and using that the gradient of the log-likelihood is zero at $z^\ast$ directly gives the asymptotic covariance of the latent posterior $ C_1 = \left(\sigma^{-2} J_{z^\ast}^T A^T A J_{z^\ast}\right)^{-1}$. A subsequent linearization of $g$ yields the claim that $\check C = J_{z^\ast}C_1 J_{z^\ast}^T$.
\end{proof}

\subsection{Inference in original space}
\label{sec:probabilistic}
As an alternative, we propose to perform the inference in the original space $\mathcal{X}$. We introduce an efficient way of solving the inverse problem by introducing a Laplace approximation for the prior.

This approach uses the actual data model \eqref{eq2.1} and employs the hierarchical prior
\begin{eqnarray}
  \label{eq:encoder-decoder2}
	x|z  &\sim& \mathcal{N}(g(z), \Gamma(z)), \nonumber \\
	z &\sim&  \mathcal{N}(0, I) ,
\end{eqnarray}
where $g$ and $\Gamma$ are given by (deep) neural networks, cf. Section 2.2.
Throughout this work, we assume that $\Gamma(z)$ is positive definite and its smallest eigenvalue is bounded away from zero, for every $z\in\mathcal Z$.
Then, the resulting posterior is given by
\begin{equation}
  \label{eq:posterior}
    \pi(x\vert y) \propto \exp\left(-\frac{1}{2\sigma^2}\|Ax-y\|_2^2\right) \pi(x),
\end{equation}
and we collect some of its properties in the following lemma.
\begin{lemma}[Variable space posterior is proper]
\label{lem:analytic posterior in original space}
The posterior~\eqref{eq:posterior} is proper, has finite $p-$th moment for $p<\infty$ and fulfills the assumptions of the Bernstein-von-Mises theorem.
\end{lemma}
\begin{proof}
Observe that the prior
\begin{equation}
    \pi(x) = \int_{\mathcal{Z}} \pi(x\vert z) \pi(z) \mathrm{d}z \propto \int_{\mathcal{Z}} \left\vert \Gamma(z)\right\vert^{-1/2} \exp\left(-\frac{(x-g(z))^T\Gamma(z)^{-1}(x-g(z))}{2} - \frac{z^Tz}{2}\right) \mathrm{d}z
\end{equation}
is proper. By $\Gamma(z)$ being positive definite, the first part of the exponent is bounded. Since $\Gamma(z)$ has its smallest eigenvalue bounded away from zero, also the determinant is bounded. Then, propriety of $\pi(x)$ follows by propriety of $\pi(z)$. By the choice of a (bounded) Gaussian prior for $\pi(z)$, also $\pi(x)$ is bounded. Then, the remaining claims follow from standard theory of Bayesian inference for linear models.
\end{proof}
An immediate consequence of the previous lemma is the fact that the MMSE estimator derived from the posterior~\eqref{eq:posterior} is consistent.
Beside this obvious statistical advantage, unfortunately, the prior is computationally infeasible, since for every evaluation an integral has to be solved.

Therefore, we make use of a linearization. A similar approach in this context has been applied for density estimation~\cite{Liue2101344118}. Here, a Laplace approximation is suggested to render the intractable prior $\pi(x)$ feasible. The Laplace approximation applies a linearization to the generator mean by taking
\begin{equation}
    g(z) \approx g(z_0) + J_{z_0}(z-z_0)
\end{equation}
for some expansion point $z_0$, which has to be determined previously, and the Jacobian $J_{z_0}\in\mathbb{R}^{d, p}$ of the generator mean at $z_0$. Similarly, we have to expand the covariance matrix $\Gamma(z)$ around $z_0$. With the argument that the variance is expected to be less volatile than the mean, it is justifiable to do a constant expansion, only. This means $\Gamma(z) \approx \Gamma(z_0)$.
Higher order expansions are in general possible, since they require higher order derivatives of the mean and covariance function with respect to $z$ which are feasible due to automatic differentiation, but the resulting posterior becomes quite unapparent.

Inserting the expansions into the prior allows for an analytic representation of $\pi(x)$ and the prior PDF can be expressed as follows.

\begin{lemma}[Laplace approximated prior]
\label{lem:stoch prior laplace}
For $z_0\in\mathcal Z$, the prior $\pi(x)$ from the hierarchical model~\eqref{eq:encoder-decoder2} is approximated by a Gaussian distribution
\begin{equation}
  \label{eq:laplace prior}
    \pi(x) \approx \pi_L(x) = \mathcal N\left(x\vert g(z_0) - J_{z_0}z_0, \Gamma(z_0) + J_{z_0}J_{z_0}^T\right).
\end{equation}
The covariance matrix $\Gamma(z_0) + J_{z_0}J_{z_0}^T$ is positive definite and thus, the prior is proper on $\mathcal{X}$.
\end{lemma}
\begin{proof}
Simple calculus yields the form of $\pi_L(x)$ and the remaining claims follow directly from the assumption that $\Gamma(z_0)$ is positive definite for every $z_0\in\mathcal{Z}$.
\end{proof}
The prior $\pi_L(x)$ is a sensible choice for solving the inverse problem and computations can be carried out analytically. Moreover, the constructed approximation will preserve important properties in the fashion of Lemma~\ref{lem:analytic posterior in original space}.
We formulate this claim in the following theorem.
\begin{theorem}
    \label{thm:laplace posterior}
    Using the Laplace approximation $\pi_L(x)$ as a prior for the Bayesian inverse problem yields a Gaussian distribution as posterior on $\mathcal X$
    \begin{equation}
      \label{eq:Laplace posterior}
        \pi_L(x\vert y) = N(x\vert \hat{x}, \hat{S})
    \end{equation}
    with covariance matrix
    \begin{equation}
        \hat{S} = \left(\sigma^{-2}A^TA + \left(\Gamma(z_0) + J_{z_0}J_{z_0}^T\right)^{-1}\right)^{-1}
    \end{equation}
    and mean 
    \begin{equation}
        \hat{x} = \hat S\left[\sigma^{-2}A^Ty + \left( \Gamma(z_0) + J_{z_0}J_{z_0}^T \right)^{-1}\left(g(z_0)-J_{z_0}z_0\right)\right].
    \end{equation}
    For this posterior, moments of arbitrary order exist and $\pi_L$ converges, as $\sigma\to 0$, to a Gaussian centered at the maximum likelihood estimate $(A^TA)^{-1}A^Ty$, with covariance given by the inverse of the Fisher information matrix $\sigma^{-2}A^TA$, for the linear model~\eqref{eq2.1}.
    \end{theorem}
\begin{proof}
The mean and covariance are results of simple calculations for linear Gaussian models and priors.
The remaining properties are clear for this Gaussian posterior. The convergence can be assessed by the Bernstein-von-Mises theorem. In particular, the posterior mean converges in probability to the maximum likelihood estimate and the covariance converges to the inverse of the Fisher information matrix of the linear model~\eqref{eq2.1}.
\end{proof}
\begin{remark}
It can be seen that the asymptotic covariance of the Laplace approximated posterior $\pi_L(x\vert y)$ is given by
\begin{equation}
    \hat{C} =  \left( \sigma^{-2} A^TA\right)^{-1}.
\end{equation}
But, a quantitative comparison to the asymptotic covariance of the $g$ push-forward latent space posterior
\begin{equation}
    \check{C} = J_{z^\ast}\left(\sigma^{-2} J_{z^\ast}^T A^T A J_{z^\ast}\right)^{-1}J_{z^\ast}^T
\end{equation}
is not trivial. In general, it is not possible to conclude that $\hat{C}$ is ``larger'' than $\check{C}$ in some sense. In fact, numerical considerations yield that the matrix $\hat C - \check C$ is often indefinite.
However, it is unambiguous from their structure that $\check{C}$ spans a linear space of dimension at most $p$. In contrast, the covariance $\hat C$ spans the whole space $\mathcal{X}$ of dimension $d$.
\end{remark}

In the end, we like to discuss the choice of the expansion point $z_0$ for the Laplace approximation.
A natural choice is to take the expansion point from the tuple $(z_0, x_0)$ which maximizes the integrand of the prior integral, \ie, the joint prior $\pi(x, z) = \pi(x\vert z) \pi(z)$. As long as $x\approx x_0$, the approximation~\eqref{eq:posterior} is reasonable. If this condition is violated, both the true prior~\eqref{eq:encoder-decoder2} and its approximation~\eqref{eq:laplace prior} are (expected to be) small, although their relative difference could be large. However, achieving this optimum in the high-dimensional space $\mathcal Z\times\mathcal X$ is numerically challenging and an optimizer is likely to get stuck in local optima. In our experiments, we achieved distinguished results using an efficient updating scheme which we will describe in the following.
\begin{enumerate}
\item Choose as initial value in $\mathcal{X}$ the solution to the linear least-squares problem
\begin{equation}
  \label{eq:starting value}
    x_0 = \argmin_{x\in\mathcal X} \sigma^{-2}\|Ax - y\|_2^2 + \|x - g(0)\|_{\Gamma(0)^{-1}}^2
\end{equation}
and for the latent vector take $z^0=f(x_0)$, where $f$ denotes the encoder mean map of the employed VAE. For more general generative models, a numerical optimization for $z^0$ can be performed which aims for the closest element in the latent space, which generates $x_0$. Since the consideration of variants of (multi-level) latent space approaches, disentanglement or style transfer are well outside the scope of this paper, we refer to~\cite{li2019disentangled,burgess2018understanding,karras2020analyzing}
\item Then, consider the log-integrand
\begin{align}
  \pi_0(x_0, z^0) &:= \log N(x_0 \vert g(z^0), \Gamma(z^0)) + \log N(z^0\vert  0, I) \\
  &\propto
  - \frac{1}{2}(z^0)^T z^0
  - \frac{1}{2}\log(\vert\Gamma(z^0)\vert) -\frac{1}{2}(x_0-g(z^0))^T\Gamma(z^0)^{-1}(x_0-g(z^0)) \nonumber \\
  &= 
  - \frac{1}{2}\left[\log(\vert\Gamma(z^0)\vert) + (z^0)^T z^0 + (x_0-g(z^0))^T\Gamma(z^0)^{-1}(x_0-g(z^0))\right] \nonumber
\end{align}
\item Evaluate the new $z^1$ by maximizing the integrand
\begin{equation}
    \label{eq:update z0}
     (I+J_{z^0}^T\Gamma(z^0)^{-1}J_{z^0}) z^1 = J_{z^0}^T\Gamma(z^0)^{-1} (x_0-g(z^0) + J_{z^0}z^0)
\end{equation}
\item If $\pi_0(x_0, z^1) > \pi_0(x_0, z^0)$ take the new value $z^1$ as expansion point candidate.
\item repeat step 3 and 4 until no further improvement is achieved.
\end{enumerate}
Note that, the choice of $x_0$ is taken in view of the data $y$. This, to some extend, renders the approach empirical Bayesian~\cite{casella1985introduction,morris1983parametric}.

\subsection{Discussion}
Recapitulating the previous sections, we compared two Bayesian inference problems based on the original space formulation~\eqref{eq2.1} and a formulation in latent space~\eqref{eq3.1}.
In literature, it is common to consider the latent space approach which simplifies the required computations to a lower dimensional space and quantities thereof are more feasible to estimate. However, this comes at the cost of an inherent bias introduced by the employed approximate statistical model. From a statistical point-of-view, it would be beneficial to be able to establish consistent Bayes estimators which asymptotically yield correct estimates also outside of the range of the generator.
Such a procedure based on the original space formulation is however numerically challenging. One possible approximative scheme is presented in the form of a Laplace approximation. This method inherits and preserves the consistency analysis of the original space approach, while being numerically feasible. In this regard, the presented approximation is expected to yield preferable solutions to the inverse problem when the information contained in the data is large.

\subsection{Generalization to unknown variance}
The natural extension of our analysis to the case of unknown variance is straightforward. We exemplary demonstrate this for the choice of an Inverse Gamma (IG) prior distribution for the variance, i.e. $\pi(\sigma^2)\propto (\sigma^2)^{-1-\alpha}\exp(-\beta/\sigma^2)$ with shape and scale hyperparameter $\alpha, \beta>0$. This often employed prior models the positive variance in a flexible manner. In fact, $\alpha$ and $\beta$ can be chosen such that $\pi(\sigma^2)$ has no finite moment. Moreover, the explicit form allows for an analytic derivation of certain quantities of interest.

For the inference in latent space the hierarchical model reads
\begin{align*}
    y\vert z, \sigma^2 &\sim \mathcal N(Ag(z), \sigma^2 I),\\
    z &\sim \pi(z) = \mathcal N(z \vert 0, I),  \\
    \sigma^2 \vert \alpha, \beta &\sim \pi(\sigma^2) = \mathrm{IG}(\sigma^2 \vert \alpha, \beta).
\end{align*}
The resulting posterior is given by
\begin{equation}
    \pi(z, \sigma^2\vert y) \propto \left(\sigma^2\right)^{-n/2}\exp\left(-\frac{1}{2\sigma^2} \|Ag(z)-y\|_2^2 - \frac{1}{2}\|z\|^2_2\right) \pi(\sigma^2).
\end{equation}
To obtain the marginal posterior for $z$, the variance needs to be integrated out. This marginalization can be done analytically, which yields
\begin{equation}
    \pi(z\vert y) \propto \int_0^\infty \left(\sigma^2\right)^{-n/2-1-\alpha}\exp\left(-\frac{\|Ag(z)-y\|_2^2 + \beta}{2\sigma^2} \right) \mathrm{d}\sigma^2\pi(z) \propto \left\{\|Ag(z) - y\|^2_2 + \beta\right\}^{-\frac{n+2\alpha}{2}} \pi(z).
\end{equation}
The marginal posterior is proper, since for $\beta>0$ the term in the brackets is bounded away from zero. Hence, the existence of moments of arbitrary order is guaranteed by the choice of a standard Gaussian distribution for $\pi(z)$.
For this marginal posterior, the claims of Lemma~\ref{lem:latent posterior} and Lemma~\ref{lem:inconsistency} follow directly with the same arguments\footnote{Similar to Remark~\ref{rem:limit}, the information limit must be understood in the sense of repeated sampling of $y$.}. This implies that, even in the case of unknown variance, the $g$ push-forward of the marginal posterior is inconsistent in the sense of derived Bayes estimators.

For the sampling distribution~\eqref{eq2.1} in the variable space $\mathcal{X}$ and the hierarchical prior of Section~\ref{sec:probabilistic}, a similar marginal posterior can be derived. In particular, the hierarchical model
\begin{align*}
    y\vert x, \sigma^2 &\sim \mathcal N(Ax, \sigma^2 I),\\
    x\vert z &\sim \mathcal{N}(g(z), \Gamma(z)), \\
    z &\sim \pi(z) = \mathcal N(z \vert 0, I), \\
    \sigma^2 \vert \alpha, \beta &\sim \pi(\sigma^2) = \mathrm{IG}(\sigma^2 \vert \alpha, \beta),
\end{align*}
yields the marginal posterior
\begin{equation}
    \pi(x \vert y) \propto \left\{\|Ax - y\|^2_2 + \beta\right\}^{-\frac{n+2\alpha}{2}} \pi(x),
\end{equation}
where the prior $\pi(x)$ is derived as in Section~\ref{sec:probabilistic}. Again, due to $\beta>0$, the same arguments of Lemma~\ref{lem:analytic posterior in original space} apply, which renders Bayes estimators based on $\pi(x\vert y)$ consistent, and $\pi(x)$ can still be approximated as in Lemma~\ref{lem:stoch prior laplace}, i.e. $\pi_L(x)\approx\pi(x)$. However, Theorem~\ref{thm:laplace posterior}, which analytically provides the posterior as a Gaussian distribution, is not directly applicable, since it relies on the fact that the posterior is given as the product of two (non-scaled) Gaussian probability density functions (PDF). Here, the posterior is slightly different and obtaining an approximative Gaussian would rely on numerical estimation of the MAP and Hessian of $\pi(x\vert y)$.

\begin{remark}
Another typical choice for the prior of the variance is the non-informative Jeffrey's prior $\pi(\sigma^2)\propto 1/\sigma^2$. In this case, ensuring propriety and existence of moments for the latent space posterior requires additional assumptions on $g$ to ensure boundedness from below of the term $\|Ag(z) - y\|_2$. A similar assumption is sufficient for the variable space.
\end{remark}

\section{Numerical examples}
\label{sec:numerics}
To validate our theoretical findings, we perform experiments on a well-known data set with linear inverse problems governed by a blurring operation and homoscedastic Gaussian noise. In particular the data model
\begin{equation}
    y \vert x \sim \mathcal{N} (Ax, \sigma^2 I)
\end{equation}
is considered with $x$ being an unknown vector in $\mathcal{X}=\mathbb{R}^{28\times 28}$ and $A\colon\mathcal X\to\mathcal Y=\mathbb{R}^{28\times 28}$ denotes a linear Gaussian blurring operator with known precision parameter $\eta>0$ which steers the impact of the blurring.
\begin{figure}[htb]
    \centering
    \begin{tikzpicture}
       \node[anchor=south west, inner sep=0] at (0,0) (image){
        \includegraphics[width=.80\linewidth]{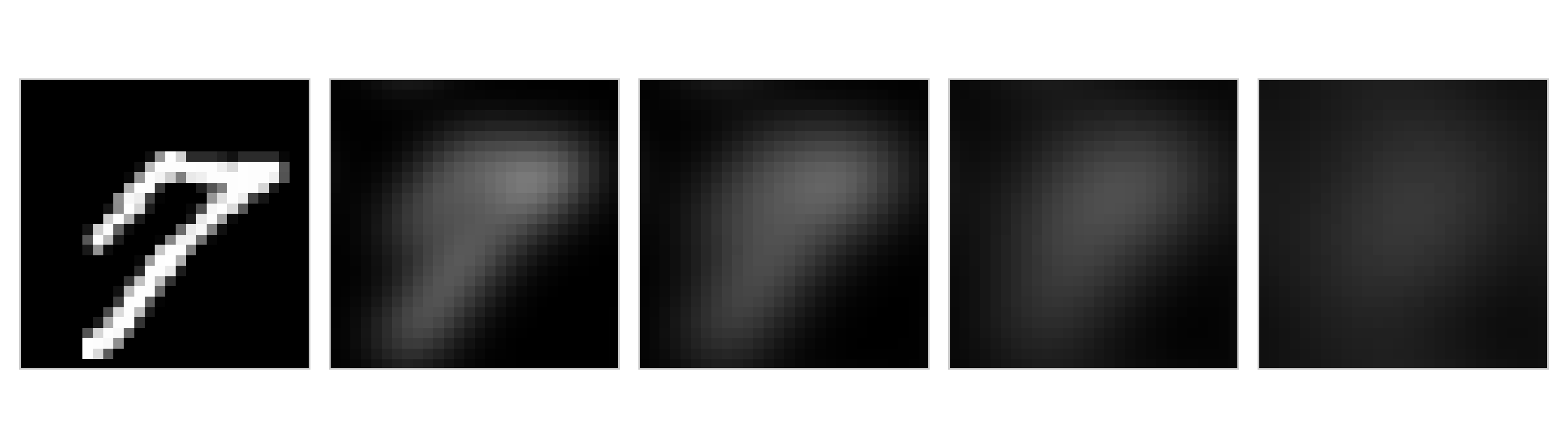}
       };
       \begin{scope}[x={(image.south east)},y={(image.north west)}]
          \draw [thick,dashed] (0.204,0.1) -- (0.204,0.9);
       \end{scope}
    \end{tikzpicture}

    \caption{Impact of the employed blurring operator on an MNIST image. On the left is the original black and white image scaled to $[0, 1]$. From left to right, the application of the blurring operator to the original image for $\eta=5, 4, 3, 2$ is shown.}
    \label{fig:blur}
\end{figure}
In Figure~\ref{fig:blur}, we show the applied blurring operation. The added noise levels in the experiments are defined for $\sigma = 10^{-s}$ for $s=1, 2, 3, 4$.

\subsection{Generative model}
We consider a typical example from machine learning. The data set of handwritten digits (MNIST)~\cite{deng2012mnist} consists of $60.000$ training sample and $10.000$ test samples, each sample corresponds to a grayscale image of size $(28, 28)$. A typical representative is shown in Figure~\ref{fig:blur}.

The generative model is chosen as an extension to the decoder of a VAE in an ``off-the-shelf'' Matlab~\cite{matlab:2021b} architecture. The latent space is chosen as $\mathcal{Z} = \mathbb{R}^{20}$ and the VAE is trained by optimizing the evidence lower bound (ELBO) on the training set using Adam with a constant learning rate $10^{-3}$ and batch-size $512$ for $20$ epochs. Afterwards, to obtain a probabilistic decoder, the final deconvolution layer of the decoder is extended to incorporate a diagonal covariance and the encoder is fixed during a transfer learning step of additional five epochs.
Example draws of the generative model are shown in Figure~\ref{fig:sketch_model} (bottom row) which resulted by taking the decoder output for some $z\sim\mathcal{N}(0, I)$. The results shown in this work are achieved using Python and Matlab. The Python source is available under \url{https://gitlab1.ptb.de/marsch02/datainformed-prior}.

\subsection{Inference with known variance}

\begin{figure}
    \centering
    \includegraphics[width=1.0\linewidth]{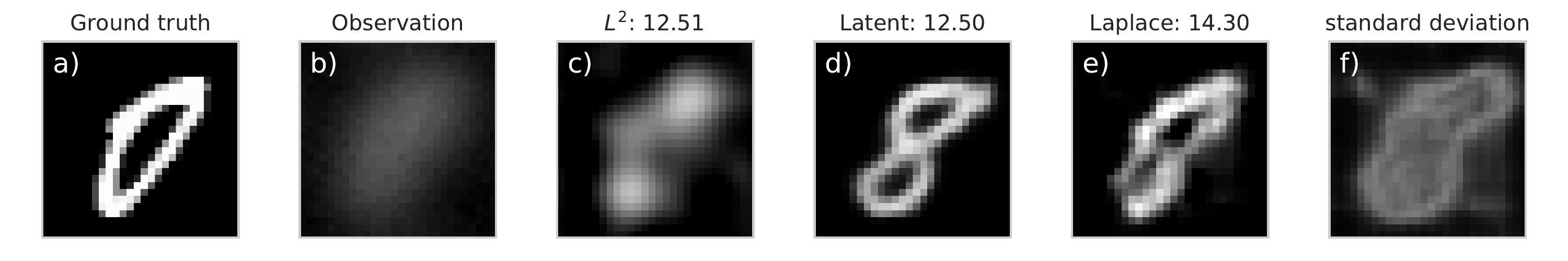}
    \caption{A single image from the MNIST test set a) is taken as ground truth for subsequent inference. The resulting observation after applying the blurring operator with $\eta=3$ and adding Gaussian noise with $\sigma=10^{-2}$ is shown in b). In c) the resulting reconstruction with the oracle $L^2$ regularized approach is given. We show in d) the reconstruction result using the Latent space approach and in e) the result of the Laplace method. Above the estimated reconstructions, the PSNR value is indicated. In f), the standard deviation of the posterior in original space is presented, \ie, the standard deviation of the marginalized posterior $\pi_L(x_i|y)$ for each pixel $i=1,\ldots, d$.}
    \label{fig:oneshot}
\end{figure}

For a quantitative comparison of the inversion quality of the \emph{latent} space approach and the \emph{Laplace} approximated variable space approach, we consider a single inference result in Figure~\ref{fig:oneshot} by taking one image of the test set as ground truth for the data model. Then, the blurring operator with $\eta=3$ is applied and Gaussian noise with $\sigma=10^{-2}$ is added. Subsequently, three approaches are applied to generate reconstruction estimates.

For the latent space approach of Section~\ref{sec:deterministic}, a BFGS optimizer is used to numerically compute the MAP of the latent space posterior $\pi(z\vert y)$ and application of $g$ yields the estimate. As starting point for the optimization in latent space, the solution to the least squares problem~\eqref{eq:starting value} is computed and the encoder mean of the VAE yields the initial guess.

For the Laplace approximated approach, the choice of $z_0$ is given in Section~\ref{sec:probabilistic}. With this expansion point, the mean of the Gaussian posterior $\pi_L(x\vert y)$ from theorem~\ref{thm:laplace posterior} is taken as the estimate.

As a reference, we additionally include a solution obtained by an $L^2$ regularized deterministic optimization of the least-squares problem
\begin{equation}
    x_{L^2}(\lambda) = \argmin_{x\in\mathcal{X}} \| Ax - y \|_2^2 + \lambda \|x\|_2^2.
\end{equation}
The regularization parameter is chosen through an oracle method by finding the value for $\lambda$ for which $x_{L^2}(\lambda)$ has the smallest difference to the ground truth in $L^2$ norm. This reference can be seen as the best possible homoscedastic Gaussian prior for the variable space approach.

As a quality measure, the peak-signal-to-noise ratio (PSNR) with the ground truth is assessed,\ie,
\begin{equation}
  \mathrm{PSNR}(x, \hat x) = 20\log_{10}(L)-10\log_{10} \|x - \hat x\|^2_2/d.
\end{equation}
Motivated by the normalized images of the MNIST data set, we take $L=1$, which renders the PSNR a rescaling of the mean-square error. However, a larger PSNR value indicates a better reconstruction result.

In Figure~\ref{fig:oneshot} it can be seen that, the oracle $L^2$ regularized approach is merely able to reconstruct the shape of the original digit. In contrast, the latent space approach yields an estimate, which resembles the digit $8$ albeit a $0$ was used as ground truth. This behavior can be explained by the closeness of the digits $0$ and $8$ in the latent space representation and the highly nonlinear, non-convex optimization problem which is to solve in $\mathcal{Z}$. In terms of PSNR, the best reconstruction is achieved using the Laplace approach in the variable space $\mathcal{X}$. Here, the blurring of the $L^2$ regularized solution is mostly resolved and the resulting digit is not collapsed into a $8$. Additionally, due to the construction of the posterior $\pi_L(x \vert y)$ and its estimate, the covariance is a byproduct. Here, we show the square root of its diagonal in Figure~\ref{fig:oneshot} f). Since the posterior is a Gaussian, this corresponds to the standard deviations of the marginalized posterior for each pixel.

\begin{figure}[htb]
    \centering
    \includegraphics[width=.9\linewidth]{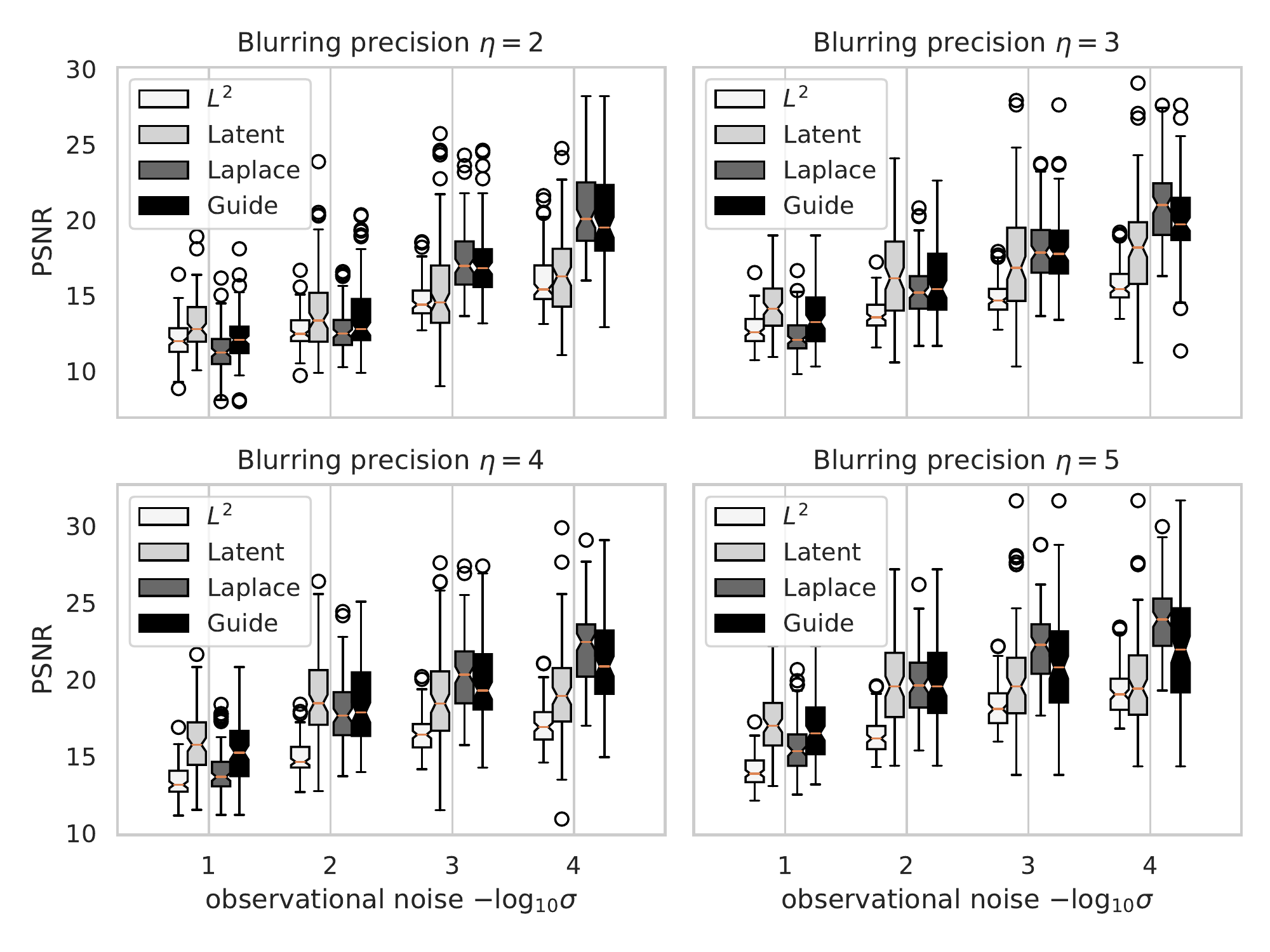}
    \caption{Peak-signal-to-noise ratio with the ground truth is shown for varying blurring precision $\eta$ in dependence on the noise variance $\sigma^2$. The negative logarithm of the standard deviation is taken as abscissa in each plot. We show the results of an oracle $L^2$ regularized solution, the results of the ``Latent'' space approach in Section~\ref{sec:deterministic}, the results of the ``Laplace'' approximated variable space approach of Section~\ref{sec:probabilistic} and a heuristic ``Guide'' as explained in Section~\ref{sec:guide}.}
    \label{fig:quality MNIST}
\end{figure}
For a statistical and comprehensive comparison, we now consider $100$ distinct images from the test set and perform the inference on every image. The results are collected in Figure~\ref{fig:quality MNIST}. For all blurring precisions, a similar behaviour of the PSNR can be observed. With decreasing noise variance, i.e. $\sigma\to 0$, the variable space approach, denoted ``Laplace'', is usually to favor over the latent space approach, denoted ``Latent''. This verifies our theoretical findings that the variable space approach is consistent. For moderately large values of $\sigma$, the latent space approach is usually to favor, since the bias introduced by the generative model is small compared to the impact of the data variance. In those cases, the latent space approach is also superior to the oracle $L^2$ regularized approach. This can be assessed for the variable space approach only for $\sigma\leq 0.01$.

\subsection{Empirical guidance}
\label{sec:guide}
We have shown that the variable space approach and the employed Laplace approximation is to favor in cases where the information contained in the data is large. However, the advice to use our approach is based on an asymptotic result and in general it is difficult to assess whether the asymptotic result is already relevant. Therefore, we propose a heuristic argument based on a simple bias estimation to give guidance on the choice of the approach.

For given $y$, both approaches yield an estimate $x_{\mathrm{Laplace}}$ and $x_{\mathrm{Latent}}$. Interpreting these estimates as new ground truth allows us to perform the inversion again using virtual data $y_{\mathrm{Laplace}}$ and $y_{\mathrm{Latent}}$. This subsequent inference yields $x_{\mathrm{Laplace}}^{\mathrm{Laplace}}$ and $x_{\mathrm{Latent}}^{\mathrm{Latent}}$, for which the squared error to their ground truth can be assessed. The method which yields the smaller deviation is to favor. In Figure~\ref{fig:quality MNIST} we employ this heuristic approach under the label ``Guide''. It is to observe that, using this guide yields, on average, better results than applying only one of the two approaches.
Considering squared errors in a cross-validation fashion, i.e. taking $x_{\mathrm{Laplace}}$ as new observation and applying the latent space approach to obtain $x_{\mathrm{Laplace}}^{\mathrm{Latent}}$ and vice versa, yields in our experiments a slightly worse guidance.

\section{Conclusions and future research}
\label{sec:outlook}
We presented approaches to the solution of linear Bayesian inverse problems using generative models as a prior.
Furthermore, we established convergence results for a state-of-the-art inference method in latent space and contrasted them to the asymptotic behaviour of an alternative approach in the high-dimensional variable space. Such an approach in the variable space is usually intractable due to the dimensionality and the complexity of the prior representation.
Therefore, we derived a novel inference technique in the variable space based on a Laplace approximation which yields an analytic posterior distribution that inherits and preserves a Bernstein-von-Mises result.
An extension to the case with unknown variance is presented and numerical examples underpin our theoretical findings.
Finally, we motivate an empirical guidance to choose between the presented approaches in real scenarios.

This proof-of-concepts work paves the way for future research tackling various questions. Extending the numerical examples to more complex and practical data sets is an important step which raises also the need for efficient numerical treatment, e.g. to tackle the involved matrix inversions. Different and more complex models for the approximation of the prior in variable space can be pursuit, e.g. using Gaussian mixture models quadrature schemes. In this regard, efficient sampling approaches in high dimension may be required to obtain samples from the posterior. Also, connections to other recent concepts, such as bilevel optimization~\cite{calatroni2017bilevel} can be of interest.

\bibliographystyle{IEEEtran}
\bibliography{datainformed_prior}

\begin{thebibliography}{10}
\providecommand{\url}[1]{#1}
\csname url@samestyle\endcsname
\providecommand{\newblock}{\relax}
\providecommand{\bibinfo}[2]{#2}
\providecommand{\BIBentrySTDinterwordspacing}{\spaceskip=0pt\relax}
\providecommand{\BIBentryALTinterwordstretchfactor}{4}
\providecommand{\BIBentryALTinterwordspacing}{\spaceskip=\fontdimen2\font plus
\BIBentryALTinterwordstretchfactor\fontdimen3\font minus
  \fontdimen4\font\relax}
\providecommand{\BIBforeignlanguage}[2]{{%
\expandafter\ifx\csname l@#1\endcsname\relax
\typeout{** WARNING: IEEEtran.bst: No hyphenation pattern has been}%
\typeout{** loaded for the language `#1'. Using the pattern for}%
\typeout{** the default language instead.}%
\else
\language=\csname l@#1\endcsname
\fi
#2}}
\providecommand{\BIBdecl}{\relax}
\BIBdecl

\bibitem{kaipio2006statistical}
J.~Kaipio and E.~Somersalo, \emph{Statistical and computational inverse
  problems}.\hskip 1em plus 0.5em minus 0.4em\relax Springer Science \&
  Business Media, 2006, vol. 160.

\bibitem{bissantz2008statistical}
N.~Bissantz and H.~Holzmann, ``Statistical inference for inverse problems,''
  \emph{Inverse Problems}, vol.~24, no.~3, p. 034009, 2008.

\bibitem{draper1998applied}
N.~R. Draper and H.~Smith, \emph{Applied regression analysis}.\hskip 1em plus
  0.5em minus 0.4em\relax John Wiley \& Sons, 1998, vol. 326.

\bibitem{smith2007spatial}
M.~Smith and L.~Fahrmeir, ``Spatial {B}ayesian variable selection with
  application to functional magnetic resonance imaging,'' \emph{Journal of the
  American Statistical Association}, vol. 102, no. 478, pp. 417--431, 2007.

\bibitem{lee2014spatial}
K.-J. Lee, G.~L. Jones, B.~S. Caffo, and S.~S. Bassett, ``Spatial {B}ayesian
  variable selection models on functional magnetic resonance imaging
  time-series data,'' \emph{Bayesian Analysis (Online)}, vol.~9, no.~3, p. 699,
  2014.

\bibitem{engl1996regularization}
H.~W. Engl, M.~Hanke, and A.~Neubauer, \emph{Regularization of inverse
  problems}.\hskip 1em plus 0.5em minus 0.4em\relax Springer Science \&
  Business Media, 1996, vol. 375.

\bibitem{robert2007bayesian}
C.~P. Robert \emph{et~al.}, \emph{The Bayesian choice: from decision-theoretic
  foundations to computational implementation}.\hskip 1em plus 0.5em minus
  0.4em\relax Springer, 2007, vol.~2.

\bibitem{gelman1995bayesian}
A.~Gelman, J.~B. Carlin, H.~S. Stern, and D.~B. Rubin, \emph{Bayesian data
  analysis}.\hskip 1em plus 0.5em minus 0.4em\relax Chapman and Hall/CRC, 1995.

\bibitem{rue2005gaussian}
H.~Rue and L.~Held, \emph{Gaussian {M}arkov random fields: theory and
  applications}.\hskip 1em plus 0.5em minus 0.4em\relax CRC press, 2005.

\bibitem{adler2018deep}
J.~Adler and O.~{\"O}ktem, ``Deep bayesian inversion,'' \emph{arXiv preprint
  arXiv:1811.05910}, 2018.

\bibitem{goodfellow2014generative}
I.~Goodfellow, J.~Pouget-Abadie, M.~Mirza, B.~Xu, D.~Warde-Farley, S.~Ozair,
  A.~Courville, and Y.~Bengio, ``Generative adversarial nets,'' \emph{Advances
  in neural information processing systems}, vol.~27, 2014.

\bibitem{kingma2019introduction}
D.~P. Kingma and M.~Welling, ``An introduction to variational autoencoders,''
  \emph{arXiv preprint arXiv:1906.02691}, 2019.

\bibitem{saito2017statistical}
Y.~Saito, S.~Takamichi, and H.~Saruwatari, ``Statistical parametric speech
  synthesis incorporating generative adversarial networks,'' \emph{IEEE/ACM
  Transactions on Audio, Speech, and Language Processing}, vol.~26, no.~1, pp.
  84--96, 2017.

\bibitem{wang2018text}
H.~Wang, Z.~Qin, and T.~Wan, ``Text generation based on generative adversarial
  nets with latent variables,'' in \emph{Pacific-Asia conference on knowledge
  discovery and data mining}.\hskip 1em plus 0.5em minus 0.4em\relax Springer,
  2018, pp. 92--103.

\bibitem{hong2019molecular}
S.~H. Hong, S.~Ryu, J.~Lim, and W.~Y. Kim, ``Molecular generative model based
  on an adversarially regularized autoencoder,'' \emph{Journal of chemical
  information and modeling}, vol.~60, no.~1, pp. 29--36, 2019.

\bibitem{albert2018modeling}
A.~Albert, E.~Strano, J.~Kaur, and M.~Gonz{\'a}lez, ``Modeling urbanization
  patterns with generative adversarial networks,'' in \emph{IGARSS 2018-2018
  IEEE International Geoscience and Remote Sensing Symposium}.\hskip 1em plus
  0.5em minus 0.4em\relax IEEE, 2018, pp. 2095--2098.

\bibitem{bai2020deep}
Y.~Bai, W.~Chen, J.~Chen, and W.~Guo, ``Deep learning methods for solving
  linear inverse problems: Research directions and paradigms,'' \emph{Signal
  Processing}, p. 107729, 2020.

\bibitem{arridge2019solving}
S.~Arridge, P.~Maass, O.~{\"O}ktem, and C.-B. Sch{\"o}nlieb, ``Solving inverse
  problems using data-driven models,'' \emph{Acta Numerica}, vol.~28, pp.
  1--174, 2019.

\bibitem{cao2018recent}
Y.-J. Cao, L.-L. Jia, Y.-X. Chen, N.~Lin, C.~Yang, B.~Zhang, Z.~Liu, X.-X. Li,
  and H.-H. Dai, ``Recent advances of generative adversarial networks in
  computer vision,'' \emph{IEEE Access}, vol.~7, pp. 14\,985--15\,006, 2018.

\bibitem{park2021review}
S.-W. Park, J.-S. Ko, J.-H. Huh, and J.-C. Kim, ``Review on generative
  adversarial networks: Focusing on computer vision and its applications,''
  \emph{Electronics}, vol.~10, no.~10, p. 1216, 2021.

\bibitem{yangjie2018review}
C.~Yangjie, J.~Lili, C.~Yongxia, L.~Nan, and L.~Xuexiang, ``Review of computer
  vision based on generative adversarial networks,'' \emph{Journal of Image and
  Graphics}, vol.~23, no.~10, pp. 1433--1449, 2018.

\bibitem{yi2019generative}
X.~Yi, E.~Walia, and P.~Babyn, ``Generative adversarial network in medical
  imaging: A review,'' \emph{Medical image analysis}, vol.~58, p. 101552, 2019.

\bibitem{10.1002/essoar.10501256.1}
\BIBentryALTinterwordspacing
Z.~Jiang, S.~Zhang, C.~Turnadge, and T.~Xu, ``Combining autoencoder neural
  network and bayesian inversion algorithms to estimate heterogeneous fracture
  permeability in enhanced geothermal reservoirs,'' \emph{Earth and Space
  Science Open Archive}, p.~19, 2019. [Online]. Available:
  \url{https://doi.org/10.1002/essoar.10501256.1}
\BIBentrySTDinterwordspacing

\bibitem{griffiths2020constrained}
R.-R. Griffiths and J.~M. Hern{\'a}ndez-Lobato, ``Constrained bayesian
  optimization for automatic chemical design using variational autoencoders,''
  \emph{Chemical science}, vol.~11, no.~2, pp. 577--586, 2020.

\bibitem{mucke2021markov}
N.~T. M{\"u}cke, B.~Sanderse, S.~Boht{\'e}, and C.~W. Oosterlee, ``Markov chain
  generative adversarial neural networks for solving bayesian inverse problems
  in physics applications,'' \emph{arXiv preprint arXiv:2111.12408}, 2021.

\bibitem{bora2017compressed}
A.~Bora, A.~Jalal, E.~Price, and A.~G. Dimakis, ``Compressed sensing using
  generative models,'' in \emph{International Conference on Machine
  Learning}.\hskip 1em plus 0.5em minus 0.4em\relax PMLR, 2017, pp. 537--546.

\bibitem{holden2021bayesian}
M.~Holden, M.~Pereyra, and K.~C. Zygalakis, ``Bayesian imaging with data-driven
  priors encoded by neural networks: Theory, methods, and algorithms,''
  \emph{arXiv preprint arXiv:2103.10182}, 2021.

\bibitem{tripp2020sample}
A.~Tripp, E.~Daxberger, and J.~M. Hern{\'a}ndez-Lobato, ``Sample-efficient
  optimization in the latent space of deep generative models via weighted
  retraining,'' \emph{Advances in Neural Information Processing Systems},
  vol.~33, 2020.

\bibitem{hussein2020image}
S.~A. Hussein, T.~Tirer, and R.~Giryes, ``Image-adaptive gan based
  reconstruction,'' in \emph{Proceedings of the AAAI Conference on Artificial
  Intelligence}, vol.~34, no.~04, 2020, pp. 3121--3129.

\bibitem{sood2018application}
R.~Sood, B.~Topiwala, K.~Choutagunta, R.~Sood, and M.~Rusu, ``An application of
  generative adversarial networks for super resolution medical imaging,'' in
  \emph{2018 17th IEEE International Conference on Machine Learning and
  Applications (ICMLA)}.\hskip 1em plus 0.5em minus 0.4em\relax IEEE, 2018, pp.
  326--331.

\bibitem{bhadra2020medical}
S.~Bhadra, W.~Zhou, and M.~A. Anastasio, ``Medical image reconstruction with
  image-adaptive priors learned by use of generative adversarial networks,'' in
  \emph{Medical Imaging 2020: Physics of Medical Imaging}, vol. 11312.\hskip
  1em plus 0.5em minus 0.4em\relax International Society for Optics and
  Photonics, 2020, p. 113120V.

\bibitem{Liue2101344118}
\BIBentryALTinterwordspacing
Q.~Liu, J.~Xu, R.~Jiang, and W.~H. Wong, ``Density estimation using deep
  generative neural networks,'' \emph{Proceedings of the National Academy of
  Sciences}, vol. 118, no.~15, 2021. [Online]. Available:
  \url{https://www.pnas.org/content/118/15/e2101344118}
\BIBentrySTDinterwordspacing

\bibitem{andrews1977digital}
H.~C. Andrews and B.~R. Hunt, \emph{Digital image restoration}.\hskip 1em plus
  0.5em minus 0.4em\relax Prentice-Hall, 1977.

\bibitem{Kofler_2020}
\BIBentryALTinterwordspacing
A.~Kofler, M.~Haltmeier, T.~Schaeffter, M.~Kachelrie{\ss}, M.~Dewey, C.~Wald,
  and C.~Kolbitsch, ``Neural networks-based regularization for large-scale
  medical image reconstruction,'' \emph{Physics in Medicine {\&} Biology},
  vol.~65, no.~13, p. 135003, jul 2020. [Online]. Available:
  \url{https://doi.org/10.1088/1361-6560/ab990e}
\BIBentrySTDinterwordspacing

\bibitem{deng2012mnist}
L.~Deng, ``The mnist database of handwritten digit images for machine learning
  research,'' \emph{IEEE Signal Processing Magazine}, vol.~29, no.~6, pp.
  141--142, 2012.

\bibitem{boynton1996linear}
G.~M. Boynton, S.~A. Engel, G.~H. Glover, and D.~J. Heeger, ``Linear systems
  analysis of functional magnetic resonance imaging in human v1,''
  \emph{Journal of Neuroscience}, vol.~16, no.~13, pp. 4207--4221, 1996.

\bibitem{richard2001fast}
M.~Richard and M.~Y.-S. Chang, ``Fast digital image inpainting,'' in
  \emph{Appeared in the Proceedings of the International Conference on
  Visualization, Imaging and Image Processing (VIIP 2001), Marbella, Spain},
  2001, pp. 106--107.

\bibitem{carasso1999linear}
A.~S. Carasso, ``Linear and nonlinear image deblurring: A documented study,''
  \emph{SIAM journal on numerical analysis}, vol.~36, no.~6, pp. 1659--1689,
  1999.

\bibitem{shorten2019survey}
C.~Shorten and T.~M. Khoshgoftaar, ``A survey on image data augmentation for
  deep learning,'' \emph{Journal of Big Data}, vol.~6, no.~1, pp. 1--48, 2019.

\bibitem{li2019disentangled}
Y.~Li, Q.~Pan, S.~Wang, H.~Peng, T.~Yang, and E.~Cambria, ``Disentangled
  variational auto-encoder for semi-supervised learning,'' \emph{Information
  Sciences}, vol. 482, pp. 73--85, 2019.

\bibitem{burgess2018understanding}
C.~P. Burgess, I.~Higgins, A.~Pal, L.~Matthey, N.~Watters, G.~Desjardins, and
  A.~Lerchner, ``Understanding disentangling in $\beta$-vae,'' \emph{arXiv
  preprint arXiv:1804.03599}, 2018.

\bibitem{karras2020analyzing}
T.~Karras, S.~Laine, M.~Aittala, J.~Hellsten, J.~Lehtinen, and T.~Aila,
  ``Analyzing and improving the image quality of stylegan,'' in
  \emph{Proceedings of the IEEE/CVF conference on computer vision and pattern
  recognition}, 2020, pp. 8110--8119.

\bibitem{casella1985introduction}
G.~Casella, ``An introduction to empirical bayes data analysis,'' \emph{The
  American Statistician}, vol.~39, no.~2, pp. 83--87, 1985.

\bibitem{morris1983parametric}
C.~N. Morris, ``Parametric empirical bayes inference: theory and
  applications,'' \emph{Journal of the American statistical Association},
  vol.~78, no. 381, pp. 47--55, 1983.

\bibitem{matlab:2021b}
MATLAB, \emph{version 9.11.0 (R2021b)}.\hskip 1em plus 0.5em minus 0.4em\relax
  Natick, Massachusetts: The MathWorks Inc., 2021.

\bibitem{calatroni2017bilevel}
L.~Calatroni, C.~Cao, J.~C. De~Los~Reyes, C.-B. Sch{\"o}nlieb, and T.~Valkonen,
  ``Bilevel approaches for learning of variational imaging models,''
  \emph{Variational Methods: In Imaging and Geometric Control}, vol.~18, no.
  252, p.~2, 2017.

\end{thebibliography}

\appendix

\section{Visual inference examples}
In this section, we highlight some details on the inference results of Section~\ref{sec:numerics}.
In particular, we consider one ground truth $x$ and fix the blurring operator $A$ with precision $\eta=4$. Then, we thoroughly analyse the performance of the presented approaches for varying noise variance $\sigma^2$. \begin{figure}[htb]
    \centering
    \begin{tikzpicture}[thick,scale=1.0, every node/.style={transform shape}]
    \node[anchor=south west] (image){
        \includegraphics[width=.9\linewidth]{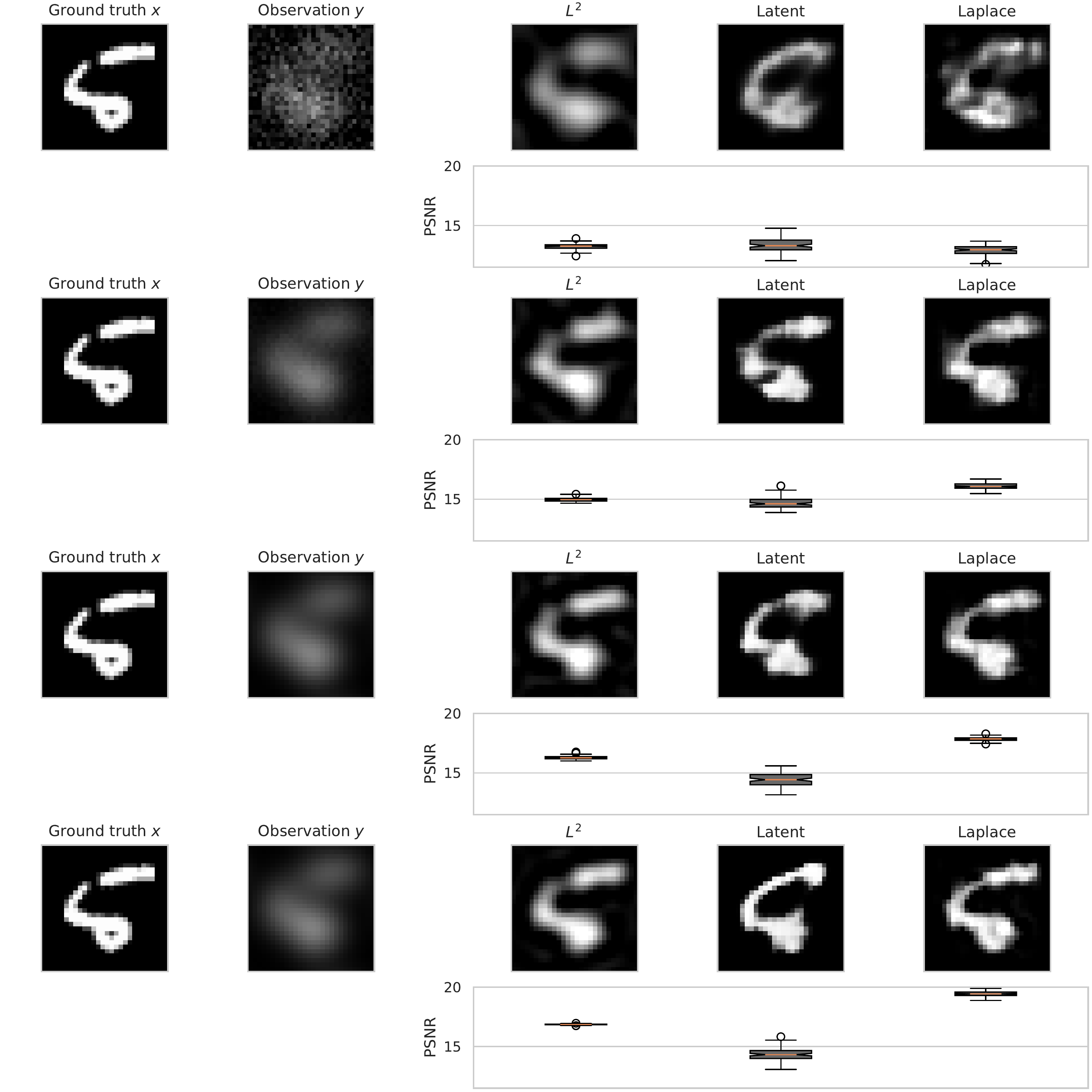}
    };
    \node[anchor=south west, above left= -3.2cm and -5cm of image, text width=3cm, align=center] (t1) {
       \small $\eta=4,\;\sigma=10^{-1}$
    };
    \node[anchor=south west, above left= -7.2cm and -5cm of image, text width=3cm, align=center] (t2) {
       \small $\eta=4,\;\sigma=10^{-2}$
    };
    \node[anchor=south west, above left= -11.2cm and -5cm of image, text width=3cm, align=center] (t3) {
       \small $\eta=4,\;\sigma=10^{-3}$
    };
    \node[anchor=south west, above left= -15.2cm and -5cm of image, text width=3cm, align=center] (t4) {
       \small $\eta=4,\;\sigma=10^{-4}$
    };

\end{tikzpicture}

    \caption{Visualization of the reconstruction quality of the different approaches. First column shows the same ground truth $x$, which is subject to blurring and noise in the second column showing the observation $y$. From top to bottom, the noise is decreased. The columns $L^2$, ``Latent'', and ``Laplace'' show one resulting estimate for each approach and below each estimate we depict the PSNR for $100$ repeated experiments, each with a different noise realization.}
    \label{fig:extended}
\end{figure}
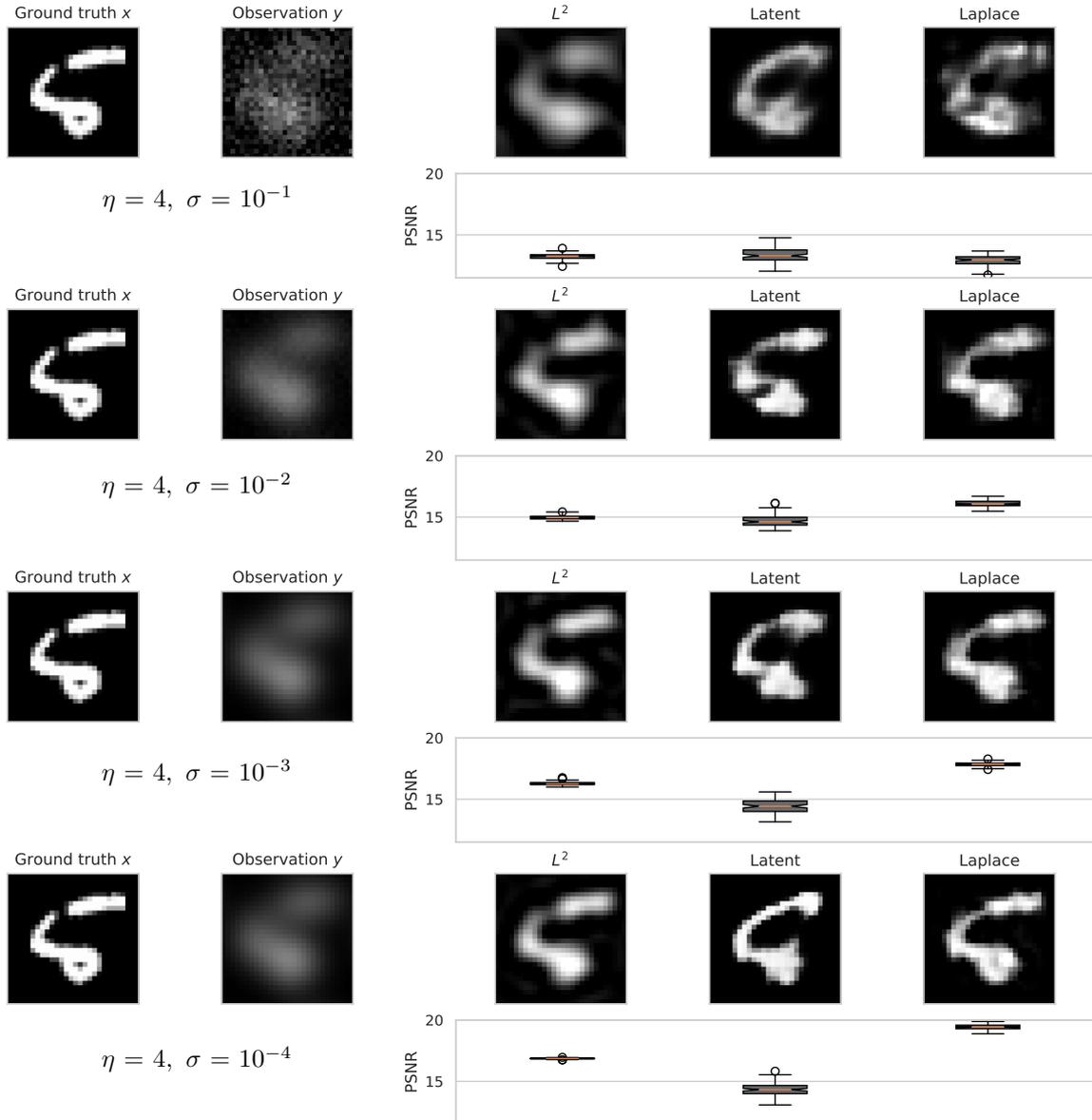
In Figure~\ref{fig:extended} we show the reconstruction quality of the $L^2$ regularized (oracle) method, the ``Latent'' space approach, and the ``Laplace'' method, side-by-side. From top to bottom, the noise variance is decreased, \ie, we show $\sigma=10^{-s}$ for $s=1, 2, 3, 4$.
The images represent one estimate for each approach and the box plots below each estimate show the distribution of PSNR values for $100$ repeated reconstructions with different noise realizations for the observation $y$. In the noisy regime (top row), the latent space approach is to favor and yields, on average, the best PSNR. For deceasing $\sigma^2$, the Laplace approach yields better results. Also, it can be observed that the latent space approach has a larger spread in the PSNR values, which indicates a strong dependence on the added noise. In contrast, the Laplace method is more stable in this regard.
\end{document}